\documentclass[twoside]{article}

\usepackage[accepted]{aistats2019}
\usepackage[utf8]{inputenc} 
\usepackage[T1]{fontenc}    
\usepackage{hyperref}       
\usepackage{url}            
\usepackage{booktabs}       
\usepackage{amsfonts}       
\usepackage{nicefrac}       
\usepackage{microtype}      
\usepackage{color}

\usepackage{algorithm}
\usepackage{algorithmic}
\usepackage{amsmath}
\usepackage{amssymb}
\usepackage{bm}
\usepackage{amsthm}
\usepackage{graphicx}
\usepackage{subfigure} 
\usepackage{wrapfig}
\usepackage{authblk}
\newtheorem{theorem}{Theorem}[section]
\newtheorem*{theorem*}{Theorem}

\begin{document}
%

%
\runningauthor{Li, Murias, Major, Dawson, Carlson}


\twocolumn[

\aistatstitle{On Target Shift in Adversarial Domain Adaptation}

\aistatsauthor{Yitong Li$^1$,\And \hspace{-1.5cm} Michael~Murias$^2$, \And \hspace{-1.5cm} Samantha Major$^3$, \And \hspace{-1.5cm} Geraldine~Dawson$^3$, \And \hspace{-1,cm} David E.~Carlson$^{1,4,5}$}
\aistatsaddress{$^1$Electrical and Computer Engineering, $^2$Duke Institute for Brain Sciences, $^3$Psychiatry and Behavioral Sciences, \\$^4$Civil and Environmental Engineering, $^5$Biostatistics and Bioinformatics, Duke University \\ \texttt{\{yitong.li, michael.murias, samantha.major, geraldine.dawson, david.carlson\}@duke.edu} } 

]

\begin{abstract}
Discrepancy between training and testing domains is a fundamental problem in the generalization of machine learning techniques. Recently, several approaches have been proposed to learn domain invariant feature representations through adversarial deep learning.  However, label shift, where the percentage of data in each class is different between domains, has received less attention. Label shift naturally arises in many contexts, especially in behavioral studies where the behaviors are freely chosen.  In this work, we propose a method called Domain Adversarial nets for Target Shift (DATS) to address label shift while learning a domain invariant representation. This is accomplished by using distribution matching to estimate label proportions in a blind test set. We extend this framework to handle multiple domains by developing a scheme to upweight source domains most similar to the target domain. Empirical results show that this framework performs well under large label shift in synthetic and real experiments, demonstrating the practical importance.
\end{abstract}

\section{Introduction}\label{sec:intro}

In supervised learning, the goal is to be able to make predictions on newly collected data (the target domain) by training on previously labeled data (the source domain). However, a gap between the source and target domains is often inevitable, due to either the changes in the data, differing data collection processes, or differing applications. Domain adaptation aims to bridge these distribution gaps to enhance generalization~\cite{TCA,ganin2016domain,mansour2009domain,zhao2017multiple}. In this manuscript, we focus on unsupervised domain adaptation, where the target samples have no labels available during training. A common approach for this scenario is to match the marginal distribution of the features without using labels ~\cite{huang2007correcting,sugiyama2008direct,gretton2009covariate}.
This is motivated by the problem of ``covariate shift,''  where the distribution of features may change, but the relationship between features and the associated outcome is constant.

In order to solve the problem of covariate shift, most existing algorithms implicitly assume that the label proportions remain unchanged~\cite{du2014semi}. However, a common case in the real world is that the percentage of samples from each class are highly variant between domains. Consider a case where we model patients in a study as separate domains.  When data is collected, the label proportions can be drastically different between patients due to many reasons, such as free behavioral choice, missing data, or differing outcomes or progression from a disease.  We will show empirically that in such a situation, these existing approaches do not help generalization due to this incorrect assumption. Similar problems also arise in anomaly rejection~\cite{scott2013classification,wen2014robust} and remote sensing image classification~\cite{tuia2015multitemporal}. This kind of problem is called class-prior change~\cite{du2014semi} or target shift~\cite{redko2018optimal}.  If an algorithm cannot account for such a shift, it can be provably suboptimal in deployment, and an overfit classifier can incorrectly remember the label proportions~\cite{tasche2017fisher}.  
Previous methods have addressed this problem by adding regularization terms~\cite{gretton2009covariate,huang2007correcting,long2015learning}.  In this manuscript, we show how the label proportions in the target domain is estimated and appropriately weight samples to correct adversarial domain adaptation methods for target shift.  


Additionally, the number of source domains is not limited to one in practice.  This necessitates explicitly accounting for multiple sources instead of treating the data as one large source domain. An unfortunate issue in multiple domain adaptation is that adding more domains is not always better.  Adding irrelevant (or less relevant) domains can hurt generalization performance~\cite{mansour2009domain}.  There has been some recent works to address choosing appropriate source domains for use in domain adaptation~\cite{zhao2017multiple,redko2018optimal}. In a similar vein, we propose a scheme to weight source domains by how similar they are to the target, allowing the domain adaptation to use only the most relevant information. This weighting can be naturally included with our previous scheme to address label imbalance. 



In this work, we propose an approach called Domain Adversarial nets for Target Shift (DATS) to address unsupervised multiple source domain adaptation with target shift. Our model is implemented in a neural network framework. First, we extend an adversarial learning scheme to get domain-invariant features~\cite{ganin2016domain} to account for label imbalance.  In these extracted features, the target label proportion is estimated by minimizing the marginal distribution gap between source and target after accounting for the known or estimated label proportions. To jointly deal with multiple sources, a weighting vector is learned to determine how much each source domain should be used. This model is trained end-to-end in an iterative way. The proposed model captures strength from related source domains while eliminating the influence from less correlated domains. Experimentally, we demonstrate on real-world data that the proposed model improves performance over numerous baselines in the presence of target shift.

\section{Notation, Background, and an Illustrative Problem}\label{sec:background}
\vspace{-1mm}
Before introducing the proposed model, the Domain Adversarial Neural Network (DANN)~\cite{ganin2016domain} framework is introduced. We will then show a simple example of how this approach does not naturally handle label shift, motivating the extensions to solve these situations.

Assuming the training/source data is given as $\{ \bm x_i, y_i, s_i\}$ for $i=1.\cdots,N_S$, where $\bm x_i$ is the input, $y_i$ is the label with values from $\{1,\cdots,L\}$, and $s_i \in \{1, \cdots, S\}$ indicates which domain the data comes from.  For domain $\mathcal{D}_s$, it contains a total of $n_s$ samples and $\sum_{s=1}^S n_s = N_S$. $S$ is the total number of source domains. The testing/target samples are given as $\{\bm x_i, s_i=T\}$ for $i=N_S+1,\cdots,N_S+N_T$ without label $y$ given. $\bm h_i = f(\bm x_i; \bm \theta_h)$ is the encoded feature of $\bm x_i$ generated by the feature extractor.  The entries of $\bm \gamma^{s} \in \Delta^{L-1}$ are the label proportions of domain $\mathcal{D}_s$, which lies on the simplex. The target domain label proportion $\bm \gamma^{T}$ is unknown, which we will later estimate. The superscript $s$ and $T$ indicate the source and target domain indexes. 

\begin{figure}[htb]
\centering
\subfigure[DATS model framework.]{
\includegraphics[width=.50\textwidth]{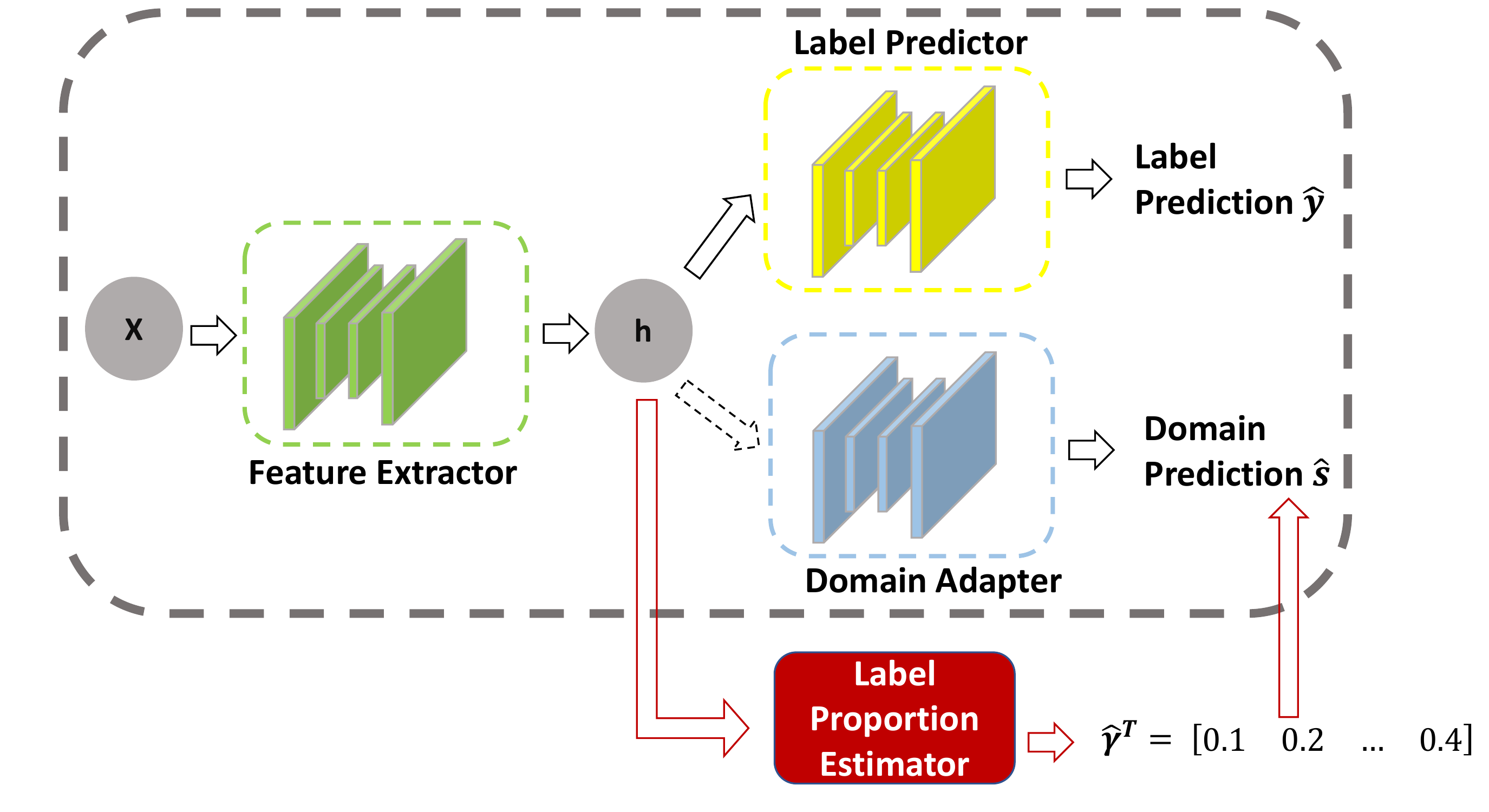}\label{fig:framework}}
\vspace{-2mm}
\subfigure[A toy example.]{
\includegraphics[width=.45\textwidth]{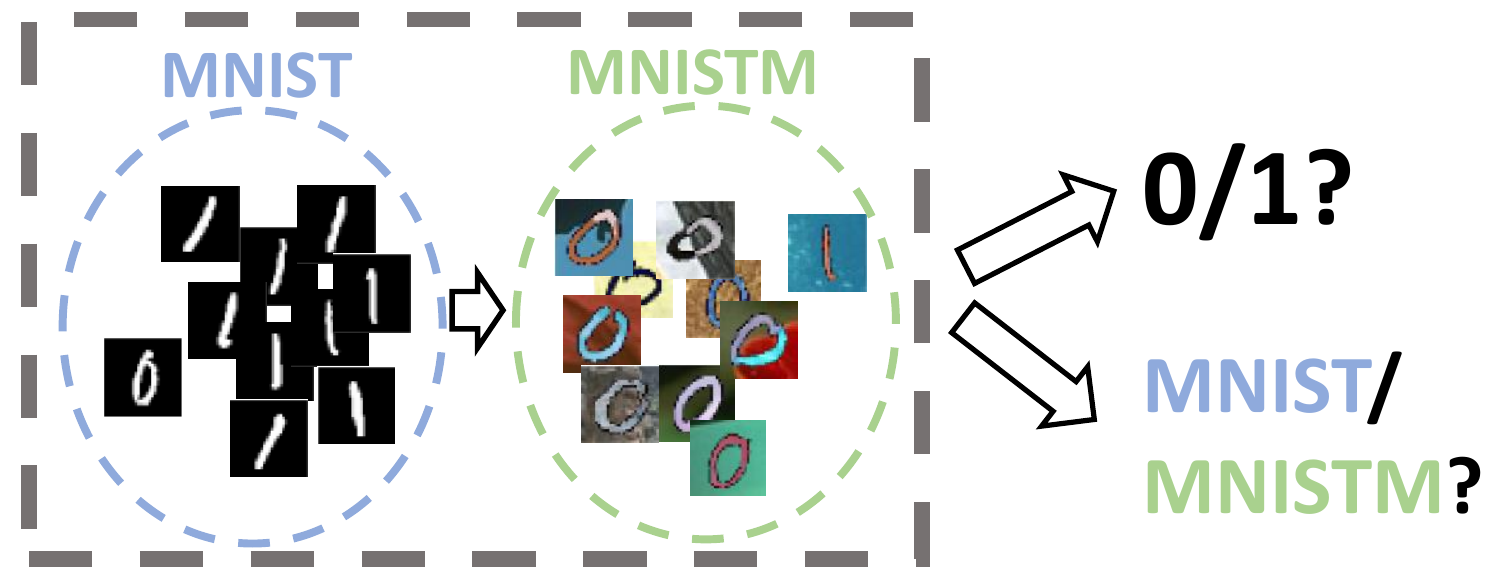}\label{fig:dann_toy} }
\vspace{-2mm}
\caption{(Top) The framework of the proposed model contains a shared feature extractor, a label predictor, a domain predictor and a label proportion estimator. The domain weighting scheme is not visualized. (Bottom) An unbalanced toy example. The proportion of digit $0$ and $1$ hugely differs in the two domains.}
\vspace{-4mm}
\end{figure}

Compared to the proposed model DATS, the framework of DANN is given in the gray dotted box in Figure~\ref{fig:framework} (everything except the red box). The intuition is to learn encoded features that can correctly predict the label while being unable to accurately predict the domain, thereby requiring that the features are domain invariant. 
DANN contains three components: a feature extractor, a label predictor, and a domain adapter. The black dotted arrow (in contrast to the solid arrows) between the extracted features and the domain adapter marks an adversarial relationship; in other words, that features $\bm h$ are expected to give low domain classification accuracy. All three components are implemented as neural networks. The feature extractor outputs features $\bm h = f(\bm x; \bm \theta_h)$ with parameters denoted as $\bm \theta_h$. $\bm h$ is then input to the label predictor with loss $\mathcal{L}_Y (y, f(\bm x; \bm \theta_h); \bm \theta_Y)$. For a classification task, $\mathcal{L}_Y$ is the cross entropy loss between the predicted and the true label pairs. For regression, this can be the mean squared error. Similarly, the domain adapter has loss $\mathcal{L}_D (s, f(\bm x; \bm \theta_h); \bm \theta_D)$ to predict which domain the data belongs to with the cross-entropy classification loss. This is adversarial as the learned features minimize the discrepancy between different domains while simultaneously maximizing label prediction accuracy. The objective function can be written as
\begin{align}\label{eq:dann_objective}
    &\hspace{-2mm}\min_{\theta_h, \theta_Y}  \max_{\theta_D} \nonumber \\ 
    &\resizebox{0.9\columnwidth}{!}{ $\mathbb{E}_{p(\mathbf{x},y,s)}\left[\mathcal{L}_Y (y, f(\bm x; \bm \theta_h); \bm \theta_Y)-
    \alpha_D \mathcal{L}_D (s, f(\bm x; \bm \theta_h); \bm \theta_D)\right].$}
\end{align}
Note that target samples are not included in the first term since the label $y$ is unknown in the target domain. $\alpha_D$ controls the relative strength of the adversary.

 DANN assumes that the domain adapter should contain no information about the label by learning features that maximize the domain loss $\mathcal{L}_D$. However, the domain adapter \emph{\textbf{must}} contain information about the label under target shift. Consider the example in Figure~\ref{fig:dann_toy} for digit image classification, where we consider domain transfer from the well-known MNIST (source) to MNISTM (target) dataset. MNISTM is a colorized version of the MNIST dataset that is used for demonstrating domain adaptation~\cite{ganin2016domain}.  Suppose that the source domain contains $10\%$ of digit 1 and $90\%$ of digit 0 while the target domain has $90\%$ of digit 1 and $10\%$ of digit 0. If a label classifier can achieve  $100\%$ accuracy, then an optimal domain predictor must be at least $90\%$ accurate.  This can be seen because the label itself is $90\%$ accurate for predicting domain, so this information must exist in the feature set.  However, the domain classifier in DANN is aims to achieve $50\%$ accuracy, which means that the learned features cannot distinguish between the domains. This contradicts the result from the naive classifier, and enforcing this condition destroys performance, which we detail empirically in Section~\ref{subsec:exp_toy}. 

In order to solve this problem, we propose the DATS model, which estimates the label distribution in the target domain to reweight data samples. This approach follows from the similar idea as balancing classes in logistic regression~\cite{gretton2009covariate}.

\section{Domain Adaptation under Target Shift}\label{sec:Model}
To address the target shift, a label proportion estimator is proposed. This is visualized as the red box in Figure~\ref{fig:framework}.  The technique for estimating the target label proportions $\bm \gamma^T$ is introduced in Section \ref{subsec:label_weight}, which is used to reweight data samples in the adversary. The red arrow in Figure~\ref{fig:framework} illustrates the usage of the label weight. The proposed method to weight multiple source domains is introduced in Section~\ref{subsec:domain_weight}. After that, a distribution matching technique is introduced to further improve the weighting accuracy in Section \ref{subsec:label_weight2}. Finally, the complete loss function and pseudo-code is covered in Section \ref{subsec:algorithm_outline}. In the following, the superscript $s=1,\cdots,S$ is the index of the source domain. For clarity, $T$ means the target domain, while  $\intercal$ means vector/matrix transpose. The label is denoted by a subscript $l \in \{1,\cdots,L \}$.


\subsection{Label Weighting Scheme}\label{subsec:label_weight}
The label proportions in the source domains are known simply by counting examples, with $\gamma^{s}_l$ representing the proportion of label $l$ in source $\mathcal{D}_s$. 
For the target domain, we propose to estimate the proportion of each label over the whole set, rather than estimating the label of each individual sample~\cite{ash2016unsupervised}. Our empirical results demonstrate that this enhances robustness.

A common assumption in target shift is that the conditional distributions from the label to the features are constant, such that $P^s(\bm x| \bm y) = P^T(\bm x| \bm y) = P(\bm x| \bm y)$ for $s=1,\dots,S$, and the variability in the joint distribution $p(\bm x, \bm y)$ is due to the shift in label proportions $p(\bm y)$~\cite{nguyen2016continuous}.  
Such an assumption is obviously untrue in the raw data for cases such as MNIST to MNISTM (see Figure~\ref{fig:dann_toy}, where the color differences in the raw data break this assumption). After correcting for the target shift and with the adversarial framework, the assumption that the feature extractor $\bm h = f(\bm x; \bm \theta_h)$ provides domain-invariant features is much more reasonable, so the assumption $P^s(\bm h| \bm y) = P^T(\bm h| \bm y) = P(\bm h| \bm y)$ is better aligned with reality.  


This assumption can be used to estimate the label proportions in the target domain via marginal distribution matching~\cite{zhang2013domain}; however, unlike previous approaches this estimation proceeds on the extracted feature space. Using known properties from the source domains and the weights on the target domain, we can reweight a source domain by labels to match the target distribution under the assumption.  For domain $\mathcal{D}_s$, this weighted distribution is given as
\begin{equation}\label{eq:mean_est}
\textstyle Q^s(\bm h) = \sum_{l=1}^L P^s(\bm h| y = l) {\gamma}^T_l .
\end{equation}
If the above assumption holds and $\bm{\gamma}^T$ is correct, then $Q^s(\bm h)$ is identical to the target distribution with $Q^s(\bm h)=P^T(\bm h)$. Therefore, one estimation strategy is to estimate $\bm \gamma^{T}$ to minimize a distance metric $d(Q^s(\bm h), P^T(\bm h))$ by jointly considering all source domains, where $d(\cdot)$ is a distance metric. 

In the literature, mean matching has proven to be a simple and effective approach to these types of problems~\cite{gretton2009covariate,huang2007correcting}. In contrast to prior work, we will perform mean matching in the extracted feature space rather than in the raw data. Eq.~\eqref{eq:mean_est} can be estimated by using sample means of the data points by $\bm M^{s} \bm \gamma^{T} $, where $\bm M^{s} = \bm M^{s}(\bm h| \bm y)$ is the concatenation of $\left[ \bm \mu^{s}(\bm h| y=1),\bm \mu^{s}(\bm h| y=2),\cdots,\bm \mu^{s}(\bm h| y = L) \right]$, the empirical sample means from the source domain $\mathcal{D}_s$. 
The target label proportion $\bm \gamma^T$ is estimated by restricting to the simplex and minimizing the loss function,
\begin{equation}\label{eq:label_proportion_loss}
\textstyle \mathcal{L}_{r_M}(\bm \gamma^T) = \sum_{s=1}^S \lambda^{s} || \bm M^{s}\bm \gamma^T - \bm \mu^{T} ||_2 ^2 
\end{equation}
$\lambda^{s}$ is defined as the domain weight that controls which source domains are used more (or less) for domain adaptation, described in Section~\ref{subsec:domain_weight}. $\bm \mu^T$ is the encoded feature mean of the target. The $L_2$ loss in~\eqref{eq:label_proportion_loss} can be replaced with a distribution loss such as the Wasserstein loss~\cite{redko2018optimal} or Maximum Mean Discrepancy loss~\cite{gretton2009covariate}, which we expand upon in Section \ref{subsec:label_weight2}. Note that \eqref{eq:label_proportion_loss} is a standard linearly constrained quadratic problem, yielding estimated target label proportions $\hat{\bm \gamma}^T$. In practice, this is updated by gradient descent in each minibatch.

Given the label proportions, it remains to correct the cross-entropy loss in the domain adversary defined in \eqref{eq:dann_objective} for the target shift.  To do this, define $\beta^{s}(y=l) = \frac{P^T(y=l)}{P^s(y=l)}=\frac{\gamma_l^T}{\gamma_l^s}$ as an unnormalized probability ratio of the target domain to domain $\mathcal{D}_s$, and $\bm \beta^s$ is the vector form across all labels in domain $\mathcal{D}_s$. $\hat{\gamma}^T$ is plugged in to get an empirical estimate $\hat{\bm \beta}^s$. 

By introducing the additional label weight, the domain adapter in Figure~\ref{fig:framework} is mathematically akin to a weighted classifier. The loss function of the domain adapter is given as
\begin{eqnarray}\label{eq:D_loss}
    \mathcal{L}_{D}(\bm \theta_D, \bm \theta_h) = \underbrace{\sum_{i=1}^{N_S} \frac{ \lambda^{s_i} \hat{\beta}^{s_i}_{y_i}}{n_{s_i}||\hat{\bm \beta}^s||_1} \mathcal{C}\left( \hat{s}_i, s_i; \bm \theta_D,  \bm \theta_h \right)}_\text{Source Samples} \nonumber \cr
   + \underbrace{\frac{1}{L N_T}\sum_{i=N_S+1}^{N_S + N_T} \mathcal{C}\left( \hat{s}_i, s_i; \bm \theta_D,  \bm \theta_h \right)}_\text{Target Samples} .
\end{eqnarray}
$\mathcal{C}(\cdot)$ is used as the cross-entropy loss between the estimated domain index and the ground truth. 
The label weight $\bm \beta$ is used for each source domain sample. $\hat{\beta}^{s_i}_{y_i}$ is the estimated label weight for sample $\bm x_i$ in domain $s_i$ with label $y_i$. $\lambda^s$ determines the importance of source domain $\mathcal{D}_s$, which will be introduced in the next section. 
The weighted version of domain loss increases the robustness of the algorithm under target shift.

We note that if the stated assumptions are true, then the proportion estimation scheme is asymptotically consistent. This is stated formally below.
\begin{theorem}[]\label{thm:mean_matching}
Assume that $P^s(\bm h| \bm y) = P^T(\bm h| \bm y) = P(\bm h| \bm y)$, the variance in the feature space is finite, and the label proportions are all non-zero.  When the number of training and testing samples goes to infinity, $\hat{\bm \gamma}^T$ is asymptotically consistent for $\bm \gamma^T$ if $ \left( \bm M^s \right)^{\intercal} \bm M^s $ is invertible for all $s$.
\end{theorem}
Note that the superscript $T$ means target, while $\intercal$ means transpose. The proof sketch of Theorem \ref{thm:mean_matching} is given in the supplemental material section~\ref{appn:proof}. This theorem strictly considers a single source domain; it is straightforward to be extended to multiple domains by the same arguments. When it is generalized to multiple source domains, the optimum values of the estimation $\hat{\bm \gamma}^T$ estimated from different domains are equal because the assumption $P(\bm h| y)$ is domain-invariant. Succinctly, a linear combination of asymptotically unbiased estimator is still asymptotically unbiased.

\subsection{Domain Weighting Scheme}\label{subsec:domain_weight}

Because irrelevant domains can harm adaptation performance~\cite{mansour2009domain}, multiple domain adaptation should primarily use information from the most similar domains.  However, which domains are relevant is unknown \textit{a priori}, so a weighting scheme was developed to determine the most relevant domains. The weight for source domain $\mathcal{D}_s$ is denoted as $\lambda^s$ in \eqref{eq:mean_est}.  This weighting scheme allows us to create a single network to perform multiple domain adaptation, rather than using a separate network for each domain (e.g. MDANs~\cite{zhao2017multiple}).

We determine the closest domains by finding the features with the best match in the domain adapter.  To define this, the last hidden layer of the domain adapter is given as $\bm z = f_D(\bm h; \bm \theta_D)$, where $f_D()$ is a neural network with parameter $\bm \theta_D$.  Note that this is \textit{not} the standard feature space.  Then the weights are
\begin{align}
\bm \lambda  = \text{softmax}( &[-
||\mathbb{E}[\bm z^1] - \mathbb{E}[\bm z^T]||_2^2, -
||\mathbb{E}[\bm z^2] - \mathbb{E}[\bm z^T]||_2^2, \cr
& \cdots,  - ||\mathbb{E}[\bm z^S] - \mathbb{E}[\bm z^T]||_2^2 ]), 
\nonumber
\end{align}
where the softmax is taken over this distance for each domain. $\bm z^s$ and $\bm z^T$ is the source and target features, respectively. 
Note that the distances can be scaled to determine the peakiness of the softmax function, but in practice the scale of 1 worked well.

We would like to note three important properties of this approach.  First, the choice of $z$ is important, because there is only a softmax function between $z$ and the prediction on domains. 
Therefore, if two domains are similar, then they are on average indistinguishable and appear the same to the domain adversary. Second, it is unnecessary to correct for the label imbalance. Because the label proportions re-weight the domain loss, the feature space at this stage has already accounted for the label imbalance. 
As an alternative approach, this weight can be estimated by the average probability that a sample in $\mathcal{D}_s$ is confused for a target sample; empirically, both strategies gave similar performance. Third, there is a positive feedback loop in this weighting scheme, which could potentially pose an issue if it is focused on unrelated domains. However, this feedback can be beneficial to narrow the focus to relevant domains. Empirically, we have only observed increased performance from this weighting, so this feedback loop does not appear to be a practical issue.


\subsection{Extending to Distribution Matching}\label{subsec:label_weight2}

Mean matching is an effective way of estimating label proportions; however, in many situations it is beneficial to match more than the first moment. This can be done with by matching the estimated target distribution $Q^s(\bm h)$ and the ground truth $P^T(\bm h)$ with an $f$-divergence~\cite{ali1966general}:
\begin{equation}\label{eq:f-distance}
\textstyle d_F(P^T(\bm h), Q^s(\bm h)) = \int P^T(\bm h) F \left( \frac{Q^s(\bm h)}{P^T(\bm h)}\right) d\bm h.
\end{equation}

While there are many forms of $f$-divergences, we choose $F(v) = (v-1)^2$ to match prior studies~\cite{du2014semi}, which can be effectively estimated using kernel functions. In this form, a lower bound of~\eqref{eq:f-distance} is $d_F(P^T(\bm h), Q^s(\bm h)) = \max_{r^s} \int Q^s (\bm h) r^s(\bm h) d \bm h - \int P^T(\bm h) \left( \frac{r^s(\bm h)^2}{2} + r^s(\bm h) \right) d \bm h$ using the Legendre-Fenchel convex duality~\cite{nguyen2010estimating}. This lower bound is maximized when function $r^s(\bm h)$ equals the density ratio  $\frac{Q^s(\bm h)}{P^T(\bm h)}$~\cite{keziou2003dual}. 
The lower bound of the $f$-divergence in~\eqref{eq:f-distance} requires the maximum over all possible functions for $r^s(\cdot)$, which is not achievable in practice. As a surrogate, we limit $r^s(\bm h)$ to a kernel space defined by grid points as
\begin{equation}\label{eq:ratio_kernel}
r^s(\bm h) = (\bm \alpha^s)^{\intercal} \bm \phi^s(\bm h). 
\end{equation}
$r^s(\bm h)$ is defined as a weighted combination of kernel functions $\bm \phi^s(\bm h)$ with parameters $\bm \alpha^s$ that will be learned. 
The kernel is evaluated as a radial basis function with respect to anchor or grid points.
In previous works~\cite{du2014semi,zhang2013domain,redko2018optimal}, all training samples are taken as grid points. However, it is impracticable to include all training samples in the kernel of a large dataset due to the complexity scaling of kernel methods. Computational efficiency can be accomplished through a variety of methods, such as pre-defining fixed grid points or randomly sampling a subset of the data points~\cite{snelson2006sparse}. For simplicity, we used grid points at the mean of conditional functions for labels and domains, which worked well empirically. 


If we substitute $r^s(\bm h)$ in~\eqref{eq:ratio_kernel} into a lower bound of~\eqref{eq:f-distance}, the $f$-divergence between $Q^s(\bm h)$ and $P^T(\bm h)$ can be approximated as 
\begin{align}\label{eq:f_distance_lowerbound}
\max_{\alpha^s} & -\frac{1}{2}(\bm \alpha^s)^{\intercal} \left[ \int P^T(\bm h) \bm \phi^s(\bm h) (\bm \phi^s(\bm h))^{\intercal} d \bm h \right] \bm \alpha^s \nonumber \\ 
& + (\bm \alpha^s)^{\intercal} \left[\int P^s(\bm h| \bm y) \bm \phi^s(\bm h) d \bm h \right] \bm \gamma^T -1,
\end{align}
where $P^s(\bm h| \bm y)$ is the concatenation of $\left[ P^s(\bm h| y = 1),\cdots, P^s(\bm h| y = L) \right]$. The derivation of~\eqref{eq:f_distance_lowerbound} is given in Supplemental Section~\ref{appen:label_matching}.  
To simplify the notation, define $\bm A=\int P^T(\bm h) \bm \phi^s(\bm h) (\bm \phi^s(\bm h))^{\intercal} d \bm h$ and $\bm B=\int P^s(\bm h| \bm y) \bm \phi^s(\bm h) d \bm h $, where the superscript domain index is omitted. The optimum $\bm \alpha^s$ in~\eqref{eq:f_distance_lowerbound} is $\bm A^{-1} \bm B \bm \gamma^T$. Remember that the goal is to minimize the $f$-divergence with respect to $\bm \gamma^T$, i.e. match distribution $Q^s(\bm h)$ and $P^T(\bm h)$. Substituting the optimum value of $\bm \alpha^s$ into~\eqref{eq:f_distance_lowerbound}, the objective of $\min_{\gamma^T} d_F(Q^s(\bm h), P^T(\bm h))$ becomes
\begin{align}\label{eq:dis_matching_objective}
    \min_{\gamma^T,\gamma^T_l\geq 0,||\gamma^T||_1=1} & -\frac{1}{2}(\bm \gamma^T)^{\intercal} \bm B^{\intercal} \bm A^{-1} \bm B \bm A^{-1} \bm B \bm \gamma^T \nonumber \\
    &+ \bm \gamma^T \bm B^{\intercal} \bm A^{-1} \bm B \bm \gamma^T .
\end{align}
Next we sill give how to estimate the integral with finite samples. By using kernel methods, $\bm A$ and $\bm B$ can be approximated as
\begin{align}\label{eq:f_distance_approx_terms}
\hat{\bm A}  &= \frac{1}{n^T} \sum_{j: s_j = T} \bm \phi^s(\bm h_j) (\bm \phi^s(\bm h_j))^{\intercal} \nonumber \\
\hat{\bm B} &= [ \frac{1}{n^s_1} \sum_{j: y_j = l, s_j = s} \bm \phi(\bm h_j), \cdots, \frac{1}{n^s_L} \sum_{j: y_j = L, s_j = s} \bm \phi(\bm h_j) ].
\end{align}
Note that $\intercal$ is matrix transpose (different from $T$). If we have a total of $S$ domains, there will be a total of $S \times (L+1)$ parameter $\alpha$'s to be learned. The total number of grid point is $L+1$, because we choose to use the label center in each domain. Since each $\bm \alpha$ is independent, the optimal $\bm \alpha^s$ in $\mathcal{D}_s$ can be written as 
\begin{equation}
\bm \hat{\alpha}^s = (\hat{\bm A} + \delta \bm I)^{-1} \hat{\bm B} \bm \gamma^T,
\end{equation}
where the identity matrix is added to ensure invertability. With this optimal $\bm \hat{\alpha}^s$, the only parameter to be optimized is $\bm \gamma^T$. Thus~\eqref{eq:dis_matching_objective} can be approximated  as
\begin{align}\label{eq:dis_finite_matching_objective}
 & \min_{\gamma^T,\gamma^T_l\geq 0,||\gamma^T||_1=1} \nonumber \\
 & -\frac{1}{2}( \hat{\bm B} \gamma^T)^{\intercal} (\hat{\bm A} + \delta \bm I)^{-1} \hat{\bm A} ( \hat{\bm B} + \delta \bm I)^{-1}  \hat{\bm B} \gamma^T \nonumber \\ 
& + ( \hat{\bm B} \gamma^T)^{\intercal} (\hat{\bm A} +\delta \bm I)^{-1} ( \hat{\bm B} \gamma^T)
\end{align}
Here we omit the superscript `$s$' of $\hat{\bm A}$ and $\hat{\bm B}$ for simplicity. Strictly, each domain $\mathcal{D}_s$ should have its own $\hat{\bm A}^s$ and $\hat{\bm B}^s$. When combining all source domains, the total matching loss function can be written as 
\begin{equation}\label{eq:loss_Lrf}
\textstyle\mathcal{L}_{r_F}(\bm \gamma^T) = \sum_{s=1}^S \lambda^s d_F(P^T(\bm h), Q^s(\bm h)),
\end{equation} 
where $d_F(P^T(\bm h), Q^s(\bm h))$ is approximated by the function in~\eqref{eq:dis_finite_matching_objective}. 

\subsection{Algorithm Outline}\label{subsec:algorithm_outline}

Combining all loss terms together, we need to jointly optimize neural network parameters $\bm \theta_h$, $\bm \theta_D$, $\bm \theta_Y$ and the target label proportion $\bm \gamma^T$. The objective function of the proposed model is given as
\begin{align}\label{eq:final_objective}
&\hspace{-2mm} \min_{\theta_h, \theta_Y, \gamma^T}  \max_{\theta_D} \nonumber \\
 & \hspace{-2mm}\resizebox{0.9\columnwidth}{!}{ $  \mathcal{L}_Y(\bm \theta_h, \bm \theta_Y) + \alpha_{\gamma}\mathcal{L}_{\gamma}(\bm \gamma^T) - \alpha_D \mathcal{L}_D (\bm \theta_h, \bm \theta_D) $.}
\end{align}
Here, $\mathcal{L}_Y(\bm \theta_h, \bm \theta_Y)$ is the standard cross-entropy label prediction loss. 
For purposes of optimization, the label estimation $\bm \gamma^T$ is considered a variable \textit{only} in $\mathcal{L}_{\gamma}(\bm \gamma^T) = \alpha_{\gamma,1} \mathcal{L}_{r_M}(\bm \gamma^T) + \alpha_{\gamma,2} \mathcal{L}_{r_F}(\bm \gamma^T)$, where $\mathcal{L}_{r_M}(\bm \gamma^T)$ is defined in~\eqref{eq:label_proportion_loss} and $\mathcal{L}_{r_F}(\bm \gamma^T)$ in~\eqref{eq:loss_Lrf}. The constraint on $\bm \gamma^T$ is satisfied by linking through a softmax function. For the other loss terms, $\bm \gamma^T$ is considered a constant. 
The label proportion estimator is also not used to update the feature extractor.

By setting $\alpha_{\gamma}$ to zero, the model loss in Eq.~\eqref{eq:final_objective} is equivalent to DANN if the label proportions do not update. (Note Eq.~\eqref{eq:dann_objective} is given in expectations while Eq.~\eqref{eq:final_objective} is over observed samples.) In our experiments, we compare two distinct strategies, the first only using mean matching, and the second using mean and distribution matching. The pseudo-code of the proposed algorithm is given in Algorithm~\ref{alg:propose}. 

\begin{algorithm}[h]
\caption{Multiple Source Domain Adaptation for Target Shift}\label{alg:propose}
\begin{algorithmic}
\STATE \textbf{Input}: Source samples $\{ \bm x_i, y_i, s_i \}_{i=1}^{N_S}$ and target samples $\{ \bm x_i, s_i \}_{i=N_S + 1}^{N_T}$.
\STATE \textbf{Output}: Classifier parameters $\bm \theta_h, \bm \theta_Y, \bm \theta_D$ and target label proportion $\bm \gamma^{T}$
\\\hrulefill
\STATE Calculate source label proportions $\bm \gamma^{s}$ for $s=1,\cdots,S$.
\STATE Initialize $\bm \gamma^{T} = [\frac{1}{L},\cdots,\frac{1}{L}]$ and $\lambda^{s} = \frac{1}{S}$.
\FOR{$iter = 1$ to $max\_iter$}
\STATE Sample a mini-batch training set.
  \STATE \textbf{\% Update Label Predictor and Feature Extractor}
  \STATE Fix $\bm \gamma^T$. Compute $\nabla \bm \theta_Y = \frac{\partial \mathcal{L}_Y}{\partial  \bm \theta_Y}  $ and $\nabla \bm \theta_h = \frac{\partial \mathcal{L}_Y}{\partial \bm \theta_h}  - \alpha_D \frac{\partial \mathcal{L}_D}{\partial \bm \theta_h}$ using source samples.  Update $\bm \theta_Y$ and $\bm \theta_h$ by gradient methods.
  \STATE \textbf{\% Update Domain Adapter}
  \STATE Update estimate of $\bm \lambda$ by exponential smoothing.
  \STATE Calculate $\bm \beta^s$ from current estimate of $\bm \gamma^T$.
  \STATE Compute $\nabla \bm \theta_D =\frac{\partial \mathcal{L}_D}{\partial \bm \theta_D}$ using weighted  source and target samples. Update $\bm \theta_D$ by gradient methods.
  \STATE \textbf{\% Update Target Label Proportion}
  \STATE Compute $\nabla \bm \gamma^T = \frac{\partial \mathcal{L}_{\gamma}(\bm \gamma^T)}{\partial \bm \gamma^T}$ using~\eqref{eq:label_proportion_loss} and ~\eqref{eq:loss_Lrf} on the mini-batch. Update $\bm \gamma^{T}$ by gradient methods.
\ENDFOR
\end{algorithmic}
\end{algorithm}

\section{Related Works}\label{sec:related_works}
First, we discuss previous works to estimate the proportion of labels in a blind test set. The most commonly used technique is based on marginal distribution matching~\cite{zhang2013domain,du2014semi,nguyen2016continuous}. A key idea is that the marginal target domain sample distribution, $P^T(\bm x)$, should match the distribution of a source domain weighted by the target label proportions. This can be estimated by integrating the joint of the source domain, $P^s(\bm x, y)$, with respect to estimated label proportions. Kernel mean matching~\cite{gretton2009covariate} is proved to be an effective technique to solve this problem, which has been extended in numerous ways~\cite{zhang2013domain,du2014semi,nguyen2016continuous}. However, using a RKHS to estimate $P^s(\bm x| y)$ suffers from the curse of dimensionality, reducing the utility in high dimensional feature space. Finally, the concept of Fisher consistency has been used to analyze several algorithms theoretically~\cite{tasche2017fisher}.

The covariate shift issue has an abundance of historical literature~\cite{shimodaira2000improving,zhang2013domain,sugiyama2008direct,huang2007correcting,wen2014robust}. This literature focuses on solving the discrepancy in conditional probability of $p(y| \bm x)$, while implicitly assuming the label distribution is the same in the source and target. In order to deal with target shift, people tend to use re-weight training samples in a given feature space~\cite{nguyen2016continuous}. Kernel methods can be used to learn weighting for each individual data point~\cite{long2015learning}, but is not feasible on big data. Domain adaptation aims to learn domain-invariant features, such as Transfer Component Analysis (TCA)~\cite{TCA} and Subspace Alignment (SA)~\cite{SA}. Recently, many works have explored how to learn a domain-invariant neural network feature extractor~\cite{long2016deep,long2015learning}, including via adversarial learning
~\cite{ganin2016domain,zhao2017multiple,ao2017fast,li2018extracting}. They can achieve domain-invariant features by playing a min-max game between a label classifier and a domain classifier. Compared with TCA and SA, neural network more naturally extends to a large scale dataset.~\cite{cao2017partial} proposes a partial domain adaptation in two domains under the an adversarial framework. However, generalizing their work to our situation is not trivial.
Based upon Generative Adversarial Networks (GANs) \cite{goodfellow2014generative}, many recent approaches have proposed to learn domain invariant features by transferring samples from source domain to the target~\cite{russo2017source,liu2016coupled,liu2017unsupervised,motiian2017few}. To the extent of our knowledge, these GAN-based frameworks have not considered target shift or multiple source domains for the domain adaptation task. 

Recently, optimal transport has been used to analyze the problem of label shift in domain adaptation~\cite{redko2018optimal}, but did not consider learning a feature extractor in conjunction with their framework. Notably, estimating terms in optimal transport is computationally expensive; accuracy of fast neural network based approximations is not guaranteed~\cite{courty2017optimal}. The target shift problem has also been addressed by using conditional properties via confusion matrix consistency~\cite{lipton2018detecting}. This approach has not been extended to multiple domains or adapted to learn domain-invariant feature. To the extent of our knowledge, this is the first work that learns domain-invariant features while adjusting for target shift.

\section{Experiments}\label{sec:experiment}

In this section, we test the proposed algorithm (DATS) on image and neural datasets. Most of the comparison methods are based on neural networks. For standard optimization based methods~\cite{du2014semi,zhang2013domain}, the required matrix inversion hinders their application to large-scale data. In the following, all benchmarked algorithms share the same feature extractor structure as the baseline model to ensure a fair comparison. Both `mean matching' and `DATS' are our proposed models for target shift. `Mean Matching' only has mean difference loss $\mathcal{L}_{r_M}$, while DATS contains both the label matching losses $\mathcal{L}_{r_M}$ and $\mathcal{L}_{r_F}$. Note that DANN~\cite{ganin2016domain} or MDANs~\cite{zhao2017multiple} can be viewed as similar models without label matching losses ($\alpha_{\gamma} = 0$), allowing close examination of the impact of the label matching.

\subsection{Synthetically Setting Properties on Toy Datasets}\label{subsec:exp_toy}

We first test our model on domain adaptation in handwritten digits where we synthetically alter the target shift between the source and target domains. The training set is MNIST, which is composed of digit `$4$' and `$9$', with label proportion of $20\%$ and $80\%$, respectively. The test set is MNISTM, which also contains digit `$4$' and `$9$' from, while the proportion of digit `4' changes from $10\%$ to $90\%$ with $10\%$ increments. These two digits are chosen intentionally because they are similar in shape. The feature extractor is composed of two convolutional layers. Deeper networks overfit in this problem~\cite{tzeng2017adversarial}. Both the domain adapter and label predictor are two-layer MLPs with softmax output. ReLU non-linearities are used. The result is given in Figure~\ref{fig:toy_ratio_result}.
\begin{figure}[htb]

\centering
\subfigure[AUC on MNISTM]{
\includegraphics[width=.4\textwidth]{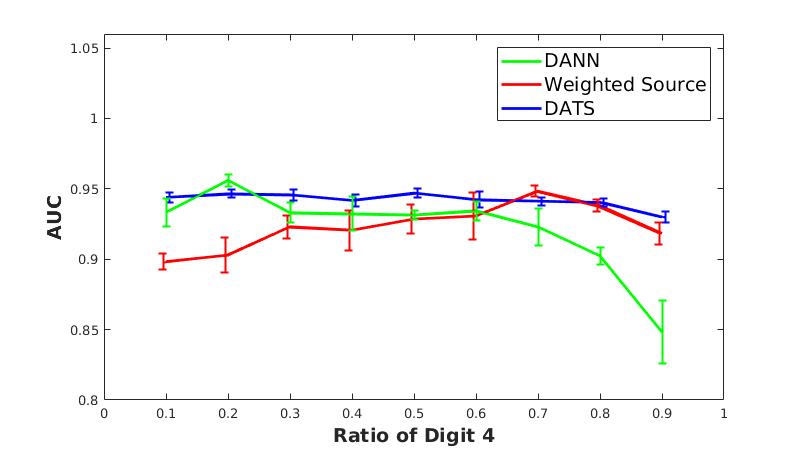}\label{fig:toy_ratio_result}}
\vspace{-3mm}
\subfigure[Label proportion in each domain.\vspace{-3mm}]{
\includegraphics[width=.4\textwidth]{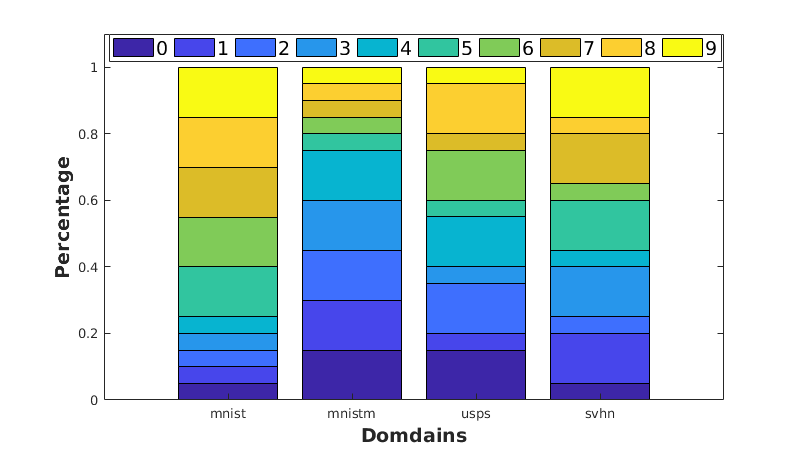}\label{fig:digit_ratio}}
\caption{(Top) Model performance comparison with different label proportion on test set. (Bottom) Label proportion in each domains for MNIST, MNISTM, USPS and SVHN.}
\end{figure}

When the target label proportion is similar to the source, the baseline DANN model performs well, because there is minimal target shift.  As the proportion of digit `$4$' increases in the target set, the amount of the target shift increases. Weighting the classes in the source set to match a uniform target label distribution, as the red line in Figure~\ref{fig:toy_ratio_result}, the performance trend is positive as the target domain becomes uniform. This is caused by the up-weighting of digit `$4$' and down-weighting of `$9$' without using any prior knowledge about the target label proportion. In comparison, the proposed algorithm robustly has high performance regardless of the label proportions. $\alpha_{\gamma}$ and $\alpha_D$ are all set as $1.0$ in this experiment. The proposed model is not overly sensitive to these tuning parameters. Note if the parameter $\alpha_D$ of \eqref{eq:dann_objective} in DANN is too large, the domain adversary becomes too powerful and predictive performance tanks due to label imbalance. Specifically, the strength of the adversary in DANN and `source weighted' is tuned to maximize performance.
As a result, the maximum AUC in DANN is above .5 because the discriminator was weakened to maximize performance (note that in practice it is not feasible to tune this parameter on an unlabeled target domain).  For our proposed models, the estimated $\hat{\bm \gamma}^T$ has at most $0.05$ difference compared to the ground truth label proportions. 

Next we look at four digit datasets: MNIST, MNISTM, USPS and SVHN. To evaluate the influence of label imbalance, we randomly assign different label proportions for each of the datasets (Figure~\ref{fig:digit_ratio}). Each time, one dataset is left out as a target while the other three are treated as training. Table~\ref{tab:image_result} gives the classification accuracy. The top row gives the name of the target domain. Note that the proposed approach robustly adapts to this situation, whereas prior methods do not. For SA~\cite{SA}, the feature input is the encoded feature $\bm h$ from baseline model for a fair comparison. 

\begin{table}[htb]
  \centering
   \resizebox{1.0\columnwidth}{!}{
  \begin{tabular}{l|c|c|c|c} 
  \hline
      & MNIST & MNISTM & USPS &SVHN \\
   \hline
      Baseline & 94.7 & 57.3 & 89.0 & 41.5 \\
      SA~\cite{SA} & 92.5 & 48.8 & 85.6 & 40.3 \\
      DAN~\cite{long2015learning} & 95.7 & 61.7 & 89.5 & 42.5 \\
      DTN~\cite{long2016deep} & 96.2 & 61.7 & 89.6 &41.7 \\
      Black Box~\cite{lipton2018detecting} & 81.5 & 42.0 & 92.4 & 42.2\\
      ADDA~\cite{tzeng2017adversarial} & 84.8 & 54.4 & 79.5 & 30.8\\
      DANN~\cite{ganin2016domain} & 94.8 & 56.6 & 89.5 & 45.0 \\
      MDANs~\cite{zhao2017multiple}  & 96.3 & 59.6 & 91.3 & 48.0 \\
      \hline
      Mean Matching & 96.6 & 67.1 & 92.3 & 47.7 \\
      DATS & \textbf{97.3} & \textbf{68.2} & \textbf{94.5} & \textbf{48.2}\\
    \hline
  \end{tabular}
  }
  \caption{Accuracy on digit image classification. } \label{tab:image_result}
  \vspace{-2mm}
\end{table}

The proposed model outperforms both DANN and MDANs on all tasks, illustrating the usefulness of the label matching term $\mathcal{L}_{\gamma}(\bm \gamma^T)$. Since the weighing scheme in MDANs does not jointly considers the label proportion, it is not robust under target shift. Practically, mean matching can stabilize the model, while adding the distribution matching marginally outperforms using only mean matching; however, even our basic strategy with minimal tuning parameters performs well compared to competing algorithms.

\subsection{Real Datasets}\label{subsec:exp_real}
We test our model on a real data composed electrical brain activity recordings using Electroencephalography (EEG) and Local Field Potentials (LFP) signals. These two datasets are described below.

\textbf{ASD Dataset}: The Autism Spectral Disorder (ASD) dataset contains Electroencephalography (EEG) signals from 22 children undergoing treatment for ASD. More details about this dataset can be found at~\cite{dawson2017autologous}. The target is their treatment stage, which is either before treatment, six months post treatment, or twelve months post treatment. The EEG signal is collected for each child when they are watching three one-minute videos designed to measure their responses to social and non-social stimuli with a standard 124 electrode layout. As is common in real-world data, the label proportions are variable, which is visualized in Appendix~\ref{appn:EEG_ratio_figure}. 

The prediction goal for this dataset is to determine when a measurement is taken. This would allow one to track how neural dynamics change as a result of treatment. Towards this end, we use the SyncNet \cite{li2017targeting} approach, which is a convolutional neural network with domain-specific interpretable features as the feature extractor. 


\begin{table}[htb]
  \centering
  \begin{tabular}{l|c|c}
  \hline
  & ASD & LFP \\
  \hline 
  SyncNet~\cite{li2017targeting} & 62.1 & 74.5 \\
   SA~\cite{SA} & 62.5 & 72.4\\
   Black Box~\cite{lipton2018detecting} & 53.6 & * \\
   DAN~\cite{long2015learning} & 61.8 &69.3 \\
   DANN~\cite{ganin2016domain}  & 63.8 & 75.1\\
   MDANs~\cite{zhao2017multiple} & 63.4 & 71.4\\
   \hline
    Mean Matching & 65.2 & \textbf{77.4} \\
    DATS & \textbf{67.2} & 77.2 \\
    \hline
  \end{tabular}
  \caption{Classification mean accuracy on EEG datasets. In our experiments,~\cite{lipton2018detecting} did not converge well on the LFP dataset.}\label{tab:Autism_result}
\end{table}

\textbf{LFP dataset}: Local Field Potential (LFP) signal are collected from implanted electrodes inside the brain. The dataset used to evaluate the proposed method contains $29$ mice from two genetic backgrounds (wild-type or CLOCK$\Delta 19$), where CLOCK$\Delta 19$ is a mouse model of bipolar disorder~\cite{van2013further}. During the data recording, each mouse spends five minutes in its home cage, spends five minutes in an open field, and ten minutes in a tail-suspension test. The task is to predict the behavior condition of the mice (home cage, open field or tail suspension). The data is pre-processed to five seconds windows. Because this dataset is controlled, its class labels are balanced. However, current experiments are being recording under freely chosen behaviors, which will result in significant target shift. In order to simulate this issue, the class labels are slightly perturbed. The label proportions for each mouse are shown in Supplemental Figure~\ref{fig:LFP_labelproportion}.

For both of the datasets, we perform leave-one-subject-out testing, i.e. one subject is picked out as target domain and the remaining ones are treated as source domains. Therefore, the source domain reaches $21$ in ASD dataset and $28$ in LFP dataset.
Mean classification accuracy over the target is given in Table~\ref{tab:Autism_result}. The proposed algorithm performs well when there is clear target shift in the data. In these experiments, the number of domains can increase drastically, while each domain usually contains only a `small' amount of data. Without adjusting for relevance of the domains, the model tends to over-fit.  The proposed model, DATS, can effectively handle adjust for label imbalance and domain weighting to give higher accuracy compared to the other baseline models. The comparative methods can fail or even not converge well when source domain number is large. Again, note that even the basic proposed strategy is effective to improve domain adaptation.

\section{Conclusion}\label{sec:conclusion}

In this work, we have addressed the target shift problem under an adversarial domain adaptation framework, and our strategy addresses is easily incorporated into standard frameworks.  We have shown that label weighting via mean matching is a simple and effective strategy, and that using distribution matching can often improve performance. Our approach also weights source domains by their relevance, increasing efficacy on multi-domain adaptation. Experiments show that the model performs consistently well in the face of large source and target domain label shift. 

\textbf{Acknowledgements}: Funding was provided by the Stylli Translational Neuroscience Award, Marcus Foundation, and NICHD P50-HD093074.  

{
\small
\bibliographystyle{plainalt}
\bibliography{ref} }

\clearpage
\appendix
\section{Derivation of the \emph{f}-divergence Lower Bound}\label{appen:label_matching}
A detailed derivation of~\eqref{eq:f_distance_lowerbound} is given here. The lower bound of the $f$-divergence is given as
\vspace{-2mm}
\begin{eqnarray*}
&& d_F(P^T(\bm h), Q^s(\bm h))  \cr
&=&  \int P^T(\bm h) F \left( \frac{Q^s(\bm h)}{P^T(\bm h)}\right) d \bm h \cr
&\geq &\resizebox{0.9\columnwidth}{!}{ $\int Q^s (\bm h) r^s(\bm h) d \bm h - \int P^T(\bm h) \left( \frac{r^s(\bm h)^2}{2} + r^s(\bm h) \right) d \bm h $}\cr
& = & -\frac{1}{2} \int P^T(\bm h) r^s(\bm h)^2 d \bm h + \int Q^s (\bm h) r^s(\bm h) d \bm h - \cr
& &\int P^T(\bm h) r_s(\bm h) d \bm h \cr
& = & -\frac{1}{2} \int P^T(\bm h) \left(\bm \alpha^s \bm \phi^s(\bm h)\right)^{\intercal} \left(\bm \alpha^s \bm \phi^s(\bm h)\right) d \bm h \cr
 & &+ \int P^s(\bm h| \bm y) \bm \gamma^T (\bm \alpha^s)^{\intercal} \bm \phi^s(\bm h) d \bm h - 1 \cr
\end{eqnarray*}
The second step follows the lower bound of $f$-divergence. The last step is derived because that $r^s(\bm h)$ is an estimation of $\frac{Q^s(\bm h)}{P^T(\bm h)}$.~\eqref{eq:dis_matching_objective} can be derived by re-arranging the terms.


\section{Proof of Theorem~\ref{thm:mean_matching}}\label{appn:proof}
In this section, we give the proof sketch of Theorem~\ref{thm:mean_matching}. For completeness, we repeat the theorem here.
\begin{theorem*}[]
Assume that $P^s(\bm h| \bm y) = P^T(\bm h| \bm y) = P(\bm h| \bm y)$, the variance in the feature space is finite, and the label proportions are all non-zero.  When the number of training and testing samples goes to infinity, $\hat{\bm \gamma}^T$ is asymptotically consistent for $\bm \gamma^T$ if $ \left( \bm M^s \right)^{\intercal} \bm M^s $ is invertible for all $s$.
\end{theorem*}
\begin{proof}
 Considering first a single source domain, the quadratic form in Equation \ref{eq:label_proportion_loss} would give that the estimator is
 \begin{equation}
\label{eq:est_basic}
 \hat{\bm \gamma}=((\bm M^s)^\intercal(\bm M^s))^{-1}(\bm M^s)^\intercal \bm \mu ^T.
 \end{equation}

From the assumption that we have finite variance and that conditional distributions are equivalent for all source and target domains, the central limit theorem gives that as the number of samples increases
 \vspace{-2mm}
\begin{equation}
 \label{eq:mus}
    (\bm \mu^s_l-\bm \mu^*_l)\sim \mathcal{N}(\bm 0,\frac{1}{n_s\gamma^s_l}\bm \Sigma_l),
\end{equation}
This is equivalent to noting that $\bm M^s$ is asymptotically centered around $\bm M^*$, because $\bm M^s$ is just the concatenation of these individual vectors.  Likewise, we have that asymptotically
\begin{equation}
\label{eq:mut}
    \bm \epsilon=(\bm \mu ^T-\bm M^* \bm \gamma^T) \sim \mathcal{N}(\bm 0 ,\sum_{l=1}^L\frac{1}{n_T\gamma^T_l}\bm \Sigma_l).
\end{equation}
Using the equality from \eqref{eq:mut} in our estimator of \eqref{eq:est_basic}, we have that
\begin{equation}
\label{eq:est}
 \hat{\bm \gamma}=((\bm M^s)^\intercal(\bm M^s))^{-1}(\bm M^s)^\intercal \bm (\bm M^* \bm \gamma^T+ \bm \epsilon).
 \end{equation}
 Note that asymptotically the errors go to 0 from \eqref{eq:mus} and \eqref{eq:mut} on all terms in \eqref{eq:est}, so the estimator has the asymptotic expectation of
 \begin{align}
 &  \lim_{n_s\rightarrow\infty, n_T\rightarrow\infty}  \mathbb{E}[\hat{\bm \gamma}^T] \nonumber \\
  & = \lim_{n_s\rightarrow\infty, n_T\rightarrow\infty} \mathbb{E}[((\bm M^s)^\intercal(\bm M^s))^{-1}(\bm M^s)^\intercal \bm (\bm M^* \bm \gamma^T+\epsilon)] \nonumber \\
  & =  ((\bm M^*)^\intercal(\bm M^*))^{-1}(\bm M^*)^\intercal  \bm M^* \bm \gamma^T=\bm \gamma^T, \nonumber
 \end{align}
if the inverse exists.
Therefore, by the weak law of large numbers, $  \lim_{n_s\rightarrow\infty, n_T\rightarrow\infty} \hat{\bm \gamma^T}=\bm \gamma^T $.
\end{proof}
\vspace{-4mm}
In two domains, $  \lim_{n_s\rightarrow\infty, n_T\rightarrow\infty} \hat{\bm \gamma^T}=\bm \gamma^T $ shows the final result of Theorem \ref{thm:mean_matching}.
\vspace{-2mm}
\section{Time Series Data Label
\vspace{-2mm}Proportion}\label{appn:EEG_ratio_figure}
The label proportion of the two EEG datasets used in experiment Section~\ref{subsec:exp_real} are visualized in Figure \ref{fig:supplement}.
\begin{figure}[htb]
\centering
\subfigure[ASD Dataset Label Proportion]{
\includegraphics[width=.45\textwidth]{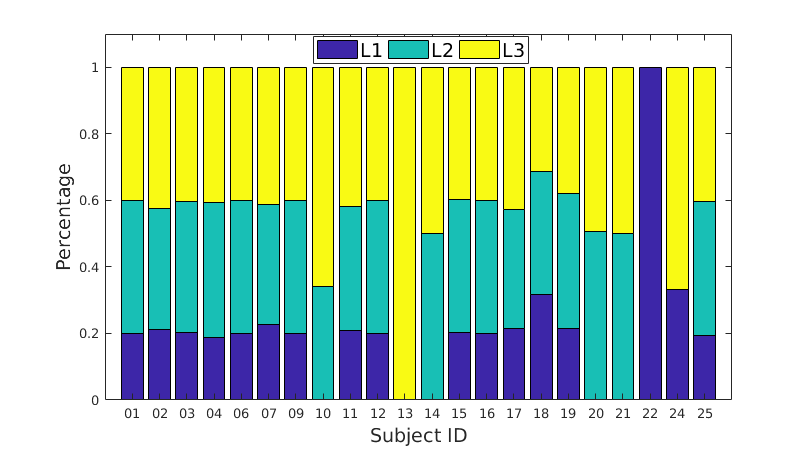} \label{fig:Autism_labelproportion}}
\vspace{-4mm}
\subfigure[LFP Dataset Label Proportion]{
\includegraphics[width=.45\textwidth]{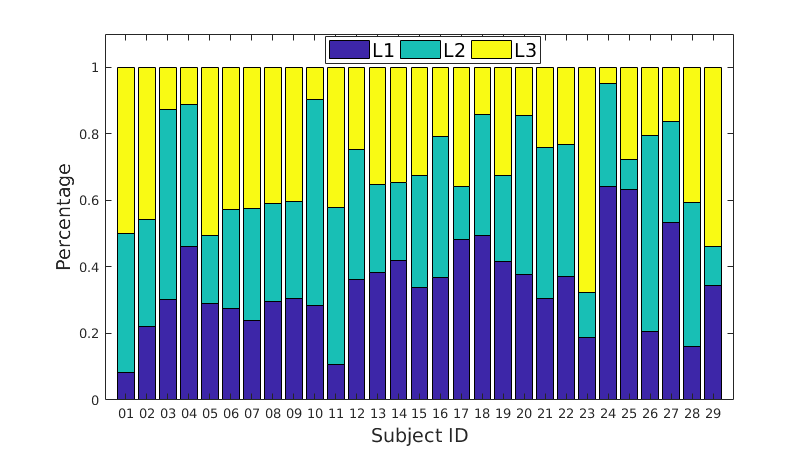} \label{fig:LFP_labelproportion}}
\caption{\label{fig:supplement}(Left) Label constitution of the ASD dataset in each domain. `L1', `L2', `L3' means before treatment, six months after treatment and twelve months after treatment, respectively. (Right) Label constitution of the LFP dataset. This dataset is for mice behavior classification using LFP signal. `L1', `L2', `L3' mean home cage, open field and tail suspension, respectively. Please refer the the text for more detailed introduction.}
\end{figure}

\end{document}